\documentclass{article} 
\usepackage{iclr2022_conference,times}

\usepackage{amsmath,amsfonts,bm}

\def\eqref#1{equation~\ref{#1}}

\def\1{\bm{1}}

\DeclareMathAlphabet{\mathsfit}{\encodingdefault}{\sfdefault}{m}{sl}
\SetMathAlphabet{\mathsfit}{bold}{\encodingdefault}{\sfdefault}{bx}{n}

\usepackage{hyperref}
\usepackage{url}
\usepackage{graphicx}
\usepackage{natbib}
\usepackage{amsthm}
\usepackage{algorithm2e}

\newtheorem{theorem}{Theorem}

\title{Kernel similarity matching with Hebbian neural networks}

\iclrfinalcopy

\author{Kyle Luther \\
Princeton University \\
\texttt{kluther@princeton.edu}
\And
H. Sebastian Seung \\
Princeton University
}

\begin{document}

\maketitle
\begin{abstract}

    Recent works have derived neural networks with online correlation-based learning rules to perform \textit{kernel similarity matching}. These works applied existing linear similarity matching algorithms to nonlinear features generated with random Fourier methods. In this paper attempt to perform kernel similarity matching by directly learning the nonlinear features. Our algorithm proceeds by deriving and then minimizing an upper bound for the sum of squared errors between output and input kernel similarities. The construction of our upper bound leads to online correlation-based learning rules which can be implemented with a 1 layer recurrent neural network. In addition to generating high-dimensional linearly separable representations, we show that our upper bound naturally yields representations which are sparse and selective for specific input patterns. We compare the approximation quality of our method to neural random Fourier method and variants of the popular but non-biological ``Nystr{\"o}m'' method for approximating the kernel matrix. Our method appears to be comparable or better than randomly sampled Nystr{\"o}m methods when the outputs are relatively low dimensional (although still potentially higher dimensional than the inputs) but less faithful when the outputs are very high dimensional.
\end{abstract}

\section{Introduction}
Brain inspired learning algorithms have a long history in the field of neural networks and machine learning \citep{rosenblatt1958perceptron, olshausen1996emergence, lee99nmf, riesenhuber1999hierarchical, hinton2007boltzmann, lillicrap2016random}. While many algorithms have diverged from their biological roots, the motivation to study biology remains clear: the human brain is such a powerful learning agent, there must be insights to be gained by making our algorithms look ``brain-like''. This paper is focused on merging biological constraints with the well-established field of kernel-based machine learning.

A common assumption in brain-inspired models of learning is that synaptic update rules should be a) online, meaning the algorithm only has access to a single input pattern at a time and b) local, meaning synapses should only be modified using information immediately available to the synapse, often just the pre- and post-firing rates of the neurons to which it is connected. Learning rules with these properties are commonly referred to as Hebbian learning rules \citep{chklovskii2016search}. 

Recent works have devised neural networks with Hebbian learning rules that perform linear similarity matching. These networks map every input $\mathbf{x}_t$ to a representation $\mathbf{y}_t$ such that linear output similarities match linear input similarities $\mathbf{y}_s \cdot \mathbf{y}_{t} \approx \mathbf{x}_s \cdot \mathbf{x}_{t}$. These networks are interesting as models for real brains because they display a number of interesting biological properties: they are recurrent networks with correlation-based learning rules \citep{pehlevan2018why} and can be modified to include non-negativity \citep{pehlevan2014nonnegative}, sparsity, and convolutional structure \citep{obeid2019structured}.

However there is a problem if one believes these networks should ultimately generate representations which are useful for downstream tasks. If similarities are actually matched, that is if $\mathbf{y}_s \cdot \mathbf{y}_{t} = \mathbf{x}_s \cdot \mathbf{x}_{t}$, then the outputs are simply an orthogonal transformation of the inputs, $\mathbf{y}_{t} = \mathbf{Q} \mathbf{x}_{t}$, which is unlikely to have significant impact on downstream tasks. \citet{bahroun2017neural} identified this problem and proposed a solution: instead of matching linear input similarities, one can match nonlinear input similarities: $\mathbf{y}_s \cdot \mathbf{y}_{t} \approx f(\mathbf{x}_s, \mathbf{x}_{t})$. The authors provided a method that can be applied to any shift-invariant kernel. They applied the random Fourier feature method of \cite{rahimi2007random} to map inputs to nonlinear feature vectors $\mathbf{x} \rightarrow \boldsymbol{\psi}$ and then applied the linear similarity matching framework of \cite{pehlevan2018why} to these nonlinear features. 

In this paper, tackle the same neural kernel similarity matching problem with a different approach. Instead of using random nonlinear features, we directly optimize for the features with Hebbian learning rules that resemble the learning rules derived in the original works on linear similarity matching. To derive our learning rules, we show that for any kernel we can upper bound the sum of squared errors $|\mathbf{y}_s \cdot \mathbf{y}_{t} - f(\mathbf{x}_s, \mathbf{x}_{t})|^2$ with a correlation-based energy. Gradient-based optimization of our upper bound with will lead to a neural network with correlation-based learning rules. 


\section{correlation-based bound for kernel similarity matching}
\textbf{Roadmap for this section} We first define the kernel similarity matching problem (Eq. \ref{eqn:cmds}). We then derive a correlation-based optimization (Eq. \ref{eqn:upper_bound}) which is an upper bound for to Eq. \ref{eqn:cmds} (up to a constant that does not depend on the representations).  We then use a Legendre transform to derive an equivalent (except for the numerical stability parameter $\lambda)$ optimization problem in Eq. \ref{eqn:final_objective} that will lend itself towards online updates. 

\textbf{Kernel similarity matching} Assume we are given a set of input vectors $\{\mathbf{x}^t \in \mathbb{R}^M\}_{t=1}^T$ and a positive semi-definite kernel function $f: \mathbb{R}^M \times \mathbb{R}^M \rightarrow \mathbb{R}$ which defines the similarity between input vectors. The goal is to find a corresponding set of representations $\{\mathbf{y}^t \in \mathbb{R}^N\}_{t=1}^T$ such that for all pairs $(s,t)$ we have $\mathbf{y}^s \cdot \mathbf{y}^t \approx f(\mathbf{x}^s, \mathbf{x}^t)$. We will assume that $T \gg N>M$: the dimensionalities of $\mathbf{x}, \mathbf{y}$ are much lower than the number of samples $T$, but $\mathbf{y}$ are still higher dimensional than the inputs. This is formalized by minimizing the sum of squared errors:
\begin{equation}
    \min_{\{\mathbf{y}^t\}} \; \frac{1}{T^2} \sum_{s,t}^T \left[f(\mathbf{x}^s,\mathbf{x}^t) - \mathbf{y}^s \cdot \mathbf{y}^t\right]^2
    \label{eqn:cmds}
\end{equation}
This is known as the classical multidimensional scaling objective \citep{borg2005cmds}. For arbitrary nonlinearity $f$ this can be solved exactly by finding the top $N$ eigenvectors of the $T \times T$ input similarity matrix \citep{borg2005cmds}, and is therefore closely related to kernel PCA \citep{scholkopf1997kernel}. However, this requires computing and storing similarities for all pairs of input vectors which breaks the online constraint that we require for biological realism. The purpose of this paper is to find an online algorithm, with correlation-based computations, that can at least approximately minimize Eq. \ref{eqn:cmds}.

\textbf{Correlation based upper bound}  In this section we provide an upper bound to Eq. \ref{eqn:cmds} which does not require computing $f(\mathbf{x}_s,\mathbf{x}_t)$ for any $(s,t)$. The first step is to expand the square in Eq (\ref{eqn:cmds}) to yield: 
\begin{equation}
    \frac{1}{T^2} \sum_{s,t}^T \left[f(\mathbf{x}^s,\mathbf{x}^t) - \mathbf{y}^s \cdot \mathbf{y}^t\right]^2 = -\frac{2}{T^2} \sum_{s,t} f(\mathbf{x}^s,\mathbf{x}^t) \mathbf{y}^s \cdot \mathbf{y}^t + \frac{1}{T^2} \sum_{s,t} (\mathbf{y}^s \cdot \mathbf{y}^t)^2 + \text{const}
    \label{eqn:reformulation}
\end{equation}
We will now show how to bound the first term on the right hand side.
\begin{theorem}\label{eq:KeyInequality}
If $f$ is a positive semidefinite kernel function, then
\begin{equation}
    \frac{1}{2T^2} \sum_{s,t} y^s y^t f(\mathbf{x}^s, \mathbf{x}^t) \geq \frac{1}{T} \sum_t q y^t f(\mathbf{x}^t, \mathbf{w}) - \frac{1}{2} q^2 f(\mathbf{w}, \mathbf{w})
    \label{eqn:inequality}
\end{equation}
for all $q$ and $\mathbf{w}$.
\end{theorem} 
\begin{proof}[Proof]
Because $f$ is a positive semi-definite kernel, we can assign to any set of $M$-dimensional vectors $\{ \mathbf{w}, \mathbf{x}_1, \hdots, \mathbf{x}_T \}$, a corresponding set of (at most) $T+1$-dimensional vectors $\{ \boldsymbol\phi_{\mathbf{w}}, \boldsymbol\phi_1, \hdots, \boldsymbol\phi_T \}$ whose inner products yield the similarity defined by $f$:
\begin{equation}
    \boldsymbol\phi_t \cdot \boldsymbol\phi_{t'} = f(\mathbf{x}_t, \mathbf{x}_{t'}) \;\;\;\;\;\; \boldsymbol\phi_t \cdot \boldsymbol\phi_{\mathbf{w}} = f(\mathbf{x}_t, \mathbf{w}) \;\;\;\;\;\; \boldsymbol\phi_{\mathbf{w}} \cdot \boldsymbol\phi_{\mathbf{w}} = f(\mathbf{w}, \mathbf{w}) 
    \label{eqn:hilbert-vectors}
\end{equation}
Now consider the vector difference $\frac{1}{T} \sum_t y_t \boldsymbol\phi_t - q \boldsymbol\phi_w$. The squared norm of this difference is of course non-negative. Additionally we can expand out this square:
\begin{equation}
    0 \leq \frac{1}{2} \left\Vert \frac{1}{T} \sum_{t} y_t \boldsymbol\phi_t - q \boldsymbol\phi_{\mathbf{w}} \right\Vert^2 = \frac{1}{2T^2} \sum_{s,t} y_s y_t \boldsymbol\phi_s \cdot \boldsymbol\phi_{t} - \frac{1}{T} \sum_t q y_t \boldsymbol\phi_t \cdot \boldsymbol\phi_{\mathbf{w}} + \frac{1}{2} q^2 \boldsymbol\phi_{\mathbf{w}} \cdot \boldsymbol\phi_{\mathbf{w}}
    \label{eqn:hilbert-norm}
\end{equation}
At this point we can simply replace all dot products with the equivalent nonlinear similarities $f(\cdot, \cdot)$ in Equation \ref{eqn:hilbert-vectors} and rearrange the terms to yield our key inequality (Eq. \ref{eqn:inequality}).
\end{proof}
 
Our inequality still holds if we maximize the right hand side with respect to $q$ and $\mathbf{w}$. For every index $i$ of $\mathbf{y}$, we find the optimal $\mathbf{w}_i, q_i$, and then replace the first pairwise sum in Eq. (\ref{eqn:reformulation}) with our upper bounds.  Additionally we rearrange the order of the summations in second term on the right hand side of Eq. (\ref{eqn:reformulation}) to yield the following upper bound for the $y$-dependent terms in Eq. (\ref{eqn:reformulation}):
\begin{equation}
    \min_{\mathbf{y}^t} \min_{q_i, \mathbf{w}_i} - \frac{1}{T} \sum_{t=1}^T \sum_{i=1}^N \left[ q_i y^t_i f(\mathbf{w}_i, \mathbf{x}^t) - \frac{1}{2} q_i^2 f(\mathbf{w}_i,\mathbf{w}_i)\right] + \frac{1}{4} \sum_{i,j=1}^N \left(\frac{1}{T} \sum_{t=1}^T y^t_i y^t_j\right)^2
    \label{eqn:upper_bound}
\end{equation}
\textbf{Online focused reformulation} We can further remove the square of the correlation matrix $\frac{1}{T} \sum_t y^t_i y^t_j$ (another impediment to online learning) by introducing a Legendre transformation: $\frac{1}{2} C_{ij}^2 \rightarrow \max_{L_{ij}} C_{ij} L_{ij} - \frac{1}{2} L_{ij}^2$:
\begin{equation}
    \min_{\mathbf{W},\mathbf{Y}, \mathbf{q}} \max_{\mathbf{L}}  \frac{1}{T} \sum_{t=1}^T \left[-\sum_{i=1}^N \left[ q_i y^t_i f(\mathbf{w}_i, \mathbf{x}^t) - \frac{1}{2} q_i^2 f(\mathbf{w}_i,\mathbf{w}_i) \right] + \frac{1}{2} \sum_{i,j=1}^N \left[ L_{ij} y^t_i y^t_j - \frac{1}{2} L_{ij}^2 \right] \right]
\end{equation}
We can swap the order of the $y$ and $L$ optimizations, because the objective obeys the strong min-max property with  $\mathbf{W},\mathbf{q}$ fixed (Appendix Section A of \citet{pehlevan2018why}). We add one final term $ \frac{\lambda}{NT} \sum_{t=1}^T \sum_{i=1}^N (y^t_i)^2 $ to the objective, which can be important for numerical stability of our resulting algorithm. In our experiments $\lambda=0.001$. Finally, to better motivate our online algorithm, we define the ``per-sample-energy'':
\begin{equation}
    e^t := \sum_{i=1}^N -\left[ q_i y^t_i f(\mathbf{w}_i, \mathbf{x}^t) - \frac{1}{2} q_i^2 f(\mathbf{w}_i,\mathbf{w}_i) \right] + \frac{1}{2} \sum_{i,j=1}^N \left[ L_{ij} y^t_i y^t_j - \frac{1}{2} L_{ij}^2 \right] + \frac{\lambda}{2} \sum_{i=1}^N (y^t_i)^2
\end{equation}
where $e^t := e(\mathbf{y}^t, \mathbf{x}^t; \mathbf{W}, \mathbf{q},\mathbf{L})$. The final optimization we will perform, which is equivalent to the optimization in Eq. \ref{eqn:upper_bound}, and is derived as an upper bound to Eq. \ref{eqn:cmds}, is thus:
\begin{equation}
    \min_{\mathbf{W}, \mathbf{q}} \max_{\mathbf{L}} \min_{\mathbf{Y}} \; \frac{1}{T} \sum_{t=1}^T e(\mathbf{y}^t, \mathbf{x}^t; \mathbf{W}, \mathbf{q},\mathbf{L})
    \label{eqn:final_objective}
\end{equation}
\section{Neural network optimization}
\begin{figure}
    \centering
    \includegraphics[width=1.0\linewidth]{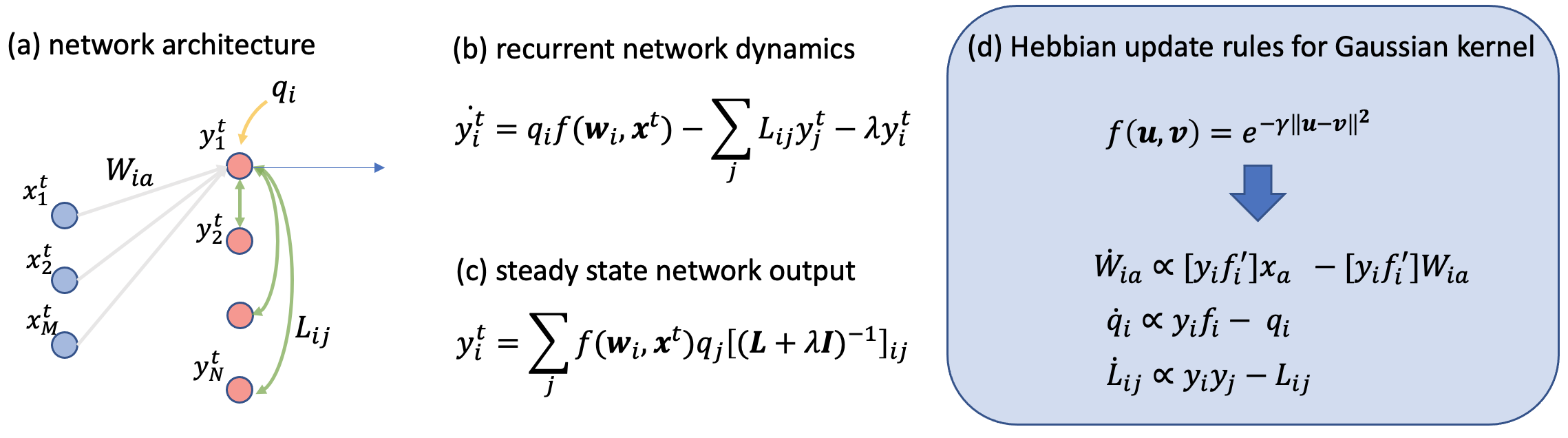}
    \caption{Neural network implementation of the optimization in Eq. \ref{eqn:final_objective} (a) network architecture (b) recurrent network dynamics (c) steady state network response (d) Hebbian update rules for the special case of Gaussian kernels (the precise form of these updates will be depend on the kernel)}
    \label{fig:network-overview}
\end{figure}

Applying a stochastic gradient descent-ascent algorithm to Eq. (\ref{eqn:final_objective}) yields a neural network (Fig. \ref{fig:network-overview}) in which $y^t_i$ is the response of neuron $i$ to input pattern $t$, $\mathbf{w}_i$ is the vector of incoming connections to neuron $i$ from the input layer, $q_i$ is a term which modulates the strength of these incoming connections, and $L_{ij}$ is a matrix of lateral recurrent connections between outputs.

Specifically the neural algorithm procedes as follows. We initialize $W_{ia} \leftarrow \mathcal{N}(0,1)$, $q_i \leftarrow 1$ and $L_{ij} \leftarrow \mathbf{I}_{ij}$. At each iteration sample a minibatch of inputs $\{\mathbf{x}^b\}$. Using Eq.\ref{eqn:network_response} we compute the $\{\mathbf{y}^b\}$ which minimize Eq. \ref{eqn:final_objective} for fixed synapses. Using these optimal $\{\mathbf{y}^b\}$, we compute the minibatch energy $e = \frac{1}{B} \sum_b e(\mathbf{x}^b,\mathbf{y}^b; \mathbf{W}, \mathbf{q},\mathbf{L})$ and take a gradient descent step for $\mathbf{q}$, a rescaled gradient descent step for $\mathbf{w}$ and a gradient ascent step for $\mathbf{L}$:
\begin{equation}
    \mathbf{w}_i \leftarrow \mathbf{w}_i - \frac{\eta_w}{q_i^2} \frac{\partial e}{\partial \mathbf{w}_i} \;\;\;\;\;\; q_i \leftarrow q_i - \eta_q \frac{\partial e}{\partial q_i} \;\;\;\;\;\; L_{ij} \leftarrow L_{ij} + \eta_l \frac{\partial e}{\partial L_{ij}}
\end{equation}

\textbf{Convergence of the neural algorithm} We treat convergence of this gradient descent ascent algorithm as an empirical issue. We adopt the ''two time scale'' strategy that has shown empirical successes for training generative adversarial networks \citep{heusel2017gans}. We choose the learning rates such that $\eta_q,\eta_w \ll \eta_l$. Intuitively when choosing $\eta_l$ to be large, the $L_{ij}$ can approximately maximize Eq. \ref{eqn:final_objective} for any particular $\mathbf{q},\mathbf{W}$ so that the min-max ordering is roughly preserved. In practice this ratio is important for convergence. We do not observe convergence when the ratios $\eta_w/ \eta_l$ or $\eta_q/\eta_l$ are large. Unfortunately it is an empirical question of what is ``too large''. If we could show that the objective were concave in $L$, it can be gradient descent ascent with smaller learning rates for $W,q$ would indeed converge to a saddle point \citep{lin2020gradient,seung2019convergence}. However, this question will have to be left for future work.

Empirically it is sometimes observed that $q_i$ quickly shrinks to a small value early in training, which subsequently leads to small gradients for $\mathbf{w}$. The rescaling of the $\mathbf{w}_i$ updates provides an adaptive learning rate that appeared to improve training times in practice. We have attached the main portion of the training code, written using PyTorch, in the appendix.

\subsection{ Network dynamics }
\label{sec:network-dynamics}
Assuming fixed parameters $\mathbf{q}, \mathbf{W}, \mathbf{L}$, the gradient for $\mathbf{y}$ can be computed for any input pattern $\mathbf{x}$. Gradient descent can be used to perform the inner loop minimization in Equation \ref{eqn:final_objective}:
\begin{equation}
    \dot{y}_i = \eta_y \left[q_i f(\mathbf{w}_i, \mathbf{x}) - \sum_{j=1}^N L_{ij} y_j - \lambda y_i \right]
\end{equation}
Like previous works on linear similarity matching, these dynamics can be interpreted as the dynamics of a one-layer recurrent neural network with all-to-all inhibition $\sum_{j=1}^N L_{ij} y_j$ between units. A diagram of this network is shown in Figure \ref{fig:network-overview}. Note that we can analytically perform the inner loop minimization with a non-neural algorithm:
\begin{equation}
    y_i \leftarrow \sum_j [\mathbf{L}+\lambda \mathbf{I}]_{ij}^{-1} q_j f(\mathbf{w}_j, \mathbf{x})
    \label{eqn:network_response}
\end{equation}
This is useful both conceptually and for speeding up the training process in our experiments. This formula shows us that $\mathbf{y}$ is a linear function of the non-linear feedforward input $f(\mathbf{w}_i, \mathbf{x}_t)$. This is different from \cite{seung2017correlation}, \cite{pehlevan2014nonnegative} where the the neurons are non-linear functions (due to non-negativity constraints) of linear feedforward input $\mathbf{w}_i \cdot \mathbf{x}_t$.

\subsection{Synaptic learning rules: arbitrary kernel}
In the previous section we saw how the $\mathbf{W}$ could be interpreted as feedforward synapses, $\mathbf{q}$ as feedforward regulatory terms, $\mathbf{L}$ as inhibitory synapses. Gradient descent on $\mathbf{W}, \mathbf{q}$ and gradient ascent on $\mathbf{L}$ provides an algorithm for performing the optimization in Equation \ref{eqn:final_objective}. At each step, we compute the optimal $\mathbf{y}$. For simplicity, we consider the case with a single input, in which case we drop the index $b$ on $\mathbf{x}^b,\mathbf{y}^b$. The stochastic gradients for $\mathbf{W}$ lead to the update: 
\begin{equation}
    \Delta \mathbf{w}_{i} \propto y_i \partial f(\mathbf{w}_i, \mathbf{x}) / \partial \mathbf{w}_i - q_i \partial f(\mathbf{w}_i, \mathbf{w}_i) / \partial \mathbf{w}_i 
    \label{eqn:dw}
\end{equation}
Classically Hebbian rules have been defined so that the update is linear in the input $\mathbf{x}$ (although they can be nonlinear in the output $\mathbf{y}$ which is a function of $\mathbf{x}$) (Eq. 1 of \citet{brito2016nonlinear}). This rule is more general as it is a nonlinear function of the input vector $\partial f(\mathbf{w}_i, \mathbf{x}) / \partial \mathbf{w}_i = h(\mathbf{x}, \mathbf{w}_i)$. However we note that the spirit of Hebb is still here as this is an online, local, correlation-based learning rule. 

The regulatory terms (essentially controlling the magnitude of the strength of feedforward input) can be updated with:
\begin{equation}
    \Delta q_{i} \propto y_i f(\mathbf{w}_i, \mathbf{x}) -  f(\mathbf{w}_i, \mathbf{w}_i) q_i 
    \label{eqn:dq}
\end{equation}
Here we have the correlation between the feedforward input and the neurons response. Finally gradients for the inhibitory synapses are:
\begin{equation}
    \Delta L_{ij} \propto y_i y_j - L_{ij}
    \label{eqn:dl}
\end{equation}
This is exactly the same ``anti-hebbian'' update seen in previous linear similarity matching works \cite{pehlevan2018why}. The inhibition grows in strength as the correlation between neurons grows. 

\subsection{Synaptic learning rules: radial basis function kernel}
Before moving on, we'll consider the form of the update rules in Eq. \ref{eqn:dw},\ref{eqn:dq} when the kernel is a radial basis function, i.e. when the kernel is a function of the Euclidean distance. For simplicity we'll also assume the kernel is normalized so that $f(\mathbf{v},\mathbf{v})=1$:
\begin{equation}
    f(\mathbf{u}, \mathbf{v}) := g(\Vert \mathbf{u} - \mathbf{v} \Vert) \;\; \text{ and } \;\; g(0) = 1
    \label{eqn:rbf-kernel}
\end{equation}
In this case we get the gradient updates for $\mathbf{w}_i,q_i$:
\begin{equation}
    \Delta \mathbf{w}_{i} \propto [y_i g'_i] \mathbf{x} - [y_i g'_i] \mathbf{w} \;\;\;\;  \Delta q_{i} \propto y_i g_i - q_i
    \label{eqn:rbf-updates}
\end{equation}
The update for $\mathbf{w}_i$ is proportional to the input $\mathbf{x}$, but modulated by the output response ($y_i$) and a function of the feedforward input ($g'_i$). The updates for $L_{ij}$ do not depend on the form of the kernel.

\section{Experiments}
We train networks using a Gaussian kernel for the half moons dataset and a ``power-cosine'' kernel (defined in section 4.2) for the MNIST dataset. We compare the approximation error (Eq. \ref{eqn:cmds}) of our method to the approximation error given by a) the optimal eigenvector-based solution (which we label as kernel PCA) b) Nystr{\"o}m approximation with uniformly sampled landmarks c) Nystr{\"o}m approximation using KMeans to generate the landmarks d) Nystr{\"o}m approximation using our generated features ($w_i$) as the landmarks and e) random Fourier feature method (This method is not applicable to the cosine-based kernel we use for the MNIST dataset). The ``dimensionality'' refers to the number of components for the PCA method, the number of landmarks for the Nystr{\"o}m methods, and the number of Fourier features for the Fourier method. See the appendix for more details regarding each of these 5 methods. Method (e) is the only other explicitly neural method.

\subsection{Half Moons Dataset}
\begin{figure}
    \centering
    \includegraphics[width=\linewidth]{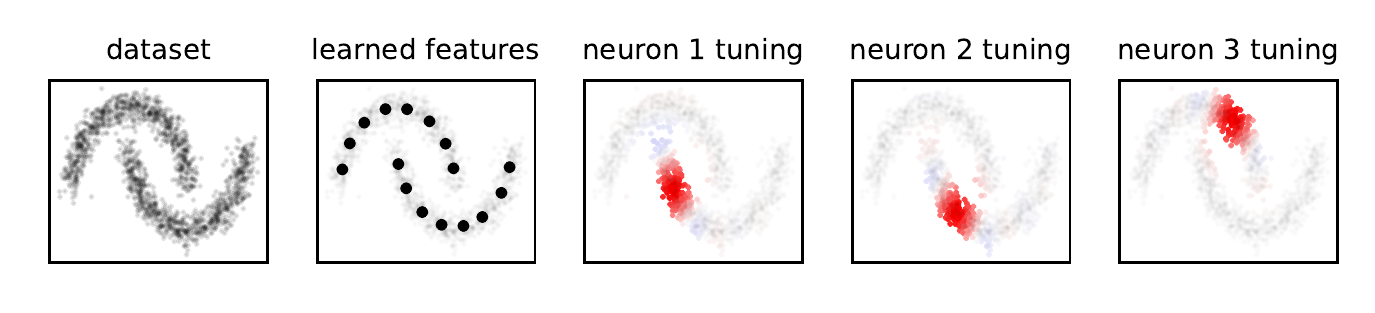}
    \caption{Overview of our Hebbian radial basis function network on the half moons dataset (a) dataset, (b) features $\{\mathbf{w}_i\}_{i=1}^{16}$ (c,d,e) response profiles of 3 neurons. }
    \label{fig:half-moons-overview}
\end{figure}
We train our algorithm on a simple half moons dataset, shown in Figure \ref{fig:half-moons-overview}. It consists of 1600 input vectors $\mathbf{x} = [x_1, x_2]$ drawn from a distribution of two noisy interleaving half circles. We use a Gaussian kernel with $\sigma=0.3$ to measure input similarities: $f(\mathbf{u}, \mathbf{v}) = e^{-\frac{\Vert \mathbf{u} - \mathbf{v} \Vert^2}{2\sigma^2}}$. We compare various number of neurons $n \in \{2,4,8,16,32,64\}$. See the appendix for training details.

\begin{figure}
    \centering
    \includegraphics[width=0.95\linewidth]{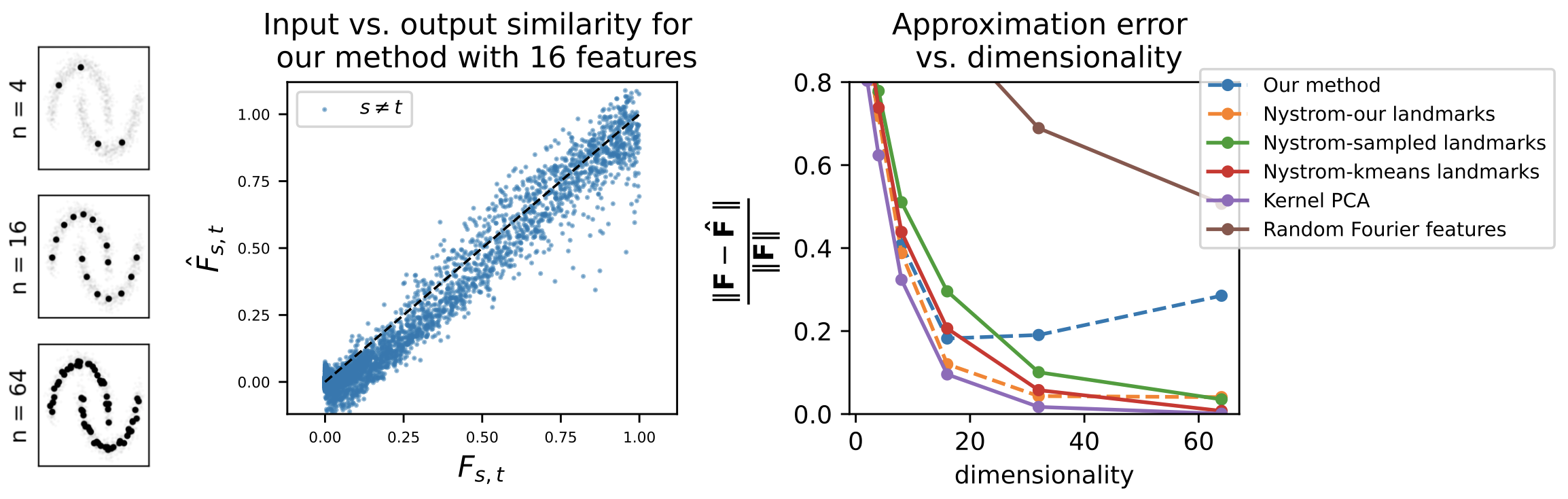}
    \caption{Approximation error vs. dimensionality for the half moons dataset (a) learned features for $n=4,16,64$ (b) input-output similarities for 16 dimensional networks (c) normalized root-mean-square error between true kernel matrix and various approximation methods. The neural method (dashed blue) that we derive in Section 3 performs well for $n <= 16$, but the approximation actually gets worse as we increase the dimensionality.}
    \label{fig:half-moons-matching}
\end{figure}

\textbf{Emergence of sparse, template-tuned neurons} In Fig. \ref{fig:half-moons-overview} we show the learned features $\{\mathbf{w}_i\}$ when we train our algorithm with 16 neurons. We observe that the features appear to tile the input space. We also show the tuning properties of 3 of the output neurons over the dataset. To generate these figures, we color code each sample in the dataset with the response of neuron $y_i$. Gray indicates zero response, red indicates a positive response and blue indicates a negative response. We observe that neurons appear to respond with large positive values centered around a small localized region of the input dataset. The features closely resemble the cluster centers return by KMeans.

\textbf{Kernel approximation error} In panel (a) of Fig. \ref{fig:half-moons-matching} we show the learned features for $n=\{4,16,64\}$. In panel (b) we plot the input similarities vs. output similarities generated by our neural algorithm with 16 dimensional outputs. In panel (c), we plot the normalized mean squared error for our method compared to the neural random Fourier method of \cite{bahroun2017neural}, non-neural Nystr{\"o}m methods, and non-neural but optimal kernel PCA method. 

We observe that for small dimensionality ($n \leq 16$) our method actually seems to marginally outperform the Nystr{\"o}m+KMeans method, which outperforms the Nystr{\"o}m+randomly sampled landmarks method. Additionally, using the Nystr{\"o}m approximation with our features seems to be uniformly better than the representations we generate with the neural net. Essentially, our algorithm leanrs useful landmarks, but for most faithful representation, it is better to just throw away the neural responses and simply use the Nystr{\"o}m approximation with our landmarks. It is worth mentioning that as you increase the dimensionality higher, the Random Rourier method ultimately does converge to zero error, unlike our method.

\textbf{Utility of representations evaluated by KMeans clustering} In Fig. \ref{fig:half-moons-clustering} we visualize the principle components of the inputs $\mathbf{x}$ and 16D representations $\mathbf{y}$. Of course, the principle components of $\mathbf{x}$ are not too interesting, they are just a reflected version of the original 2D dataset. The top two components of $\mathbf{y}$ appear to be more linearly separable than the inputs and indicate that a strong nonlinear transformation has occured. Additionally, we run KMeans on $\mathbf{x}$ and on $\mathbf{y}$ (we use the implementation of scikit-learn, and take the lowest energy solution using 100 inits). We observe that the clustering yields the nearly perfect labels when performed on $\mathbf{y}$. The kernel similarity matching vectors appear to be better suited for downstream learning tasks than the original inputs.
\begin{figure}
    \centering
    \includegraphics[width=\linewidth]{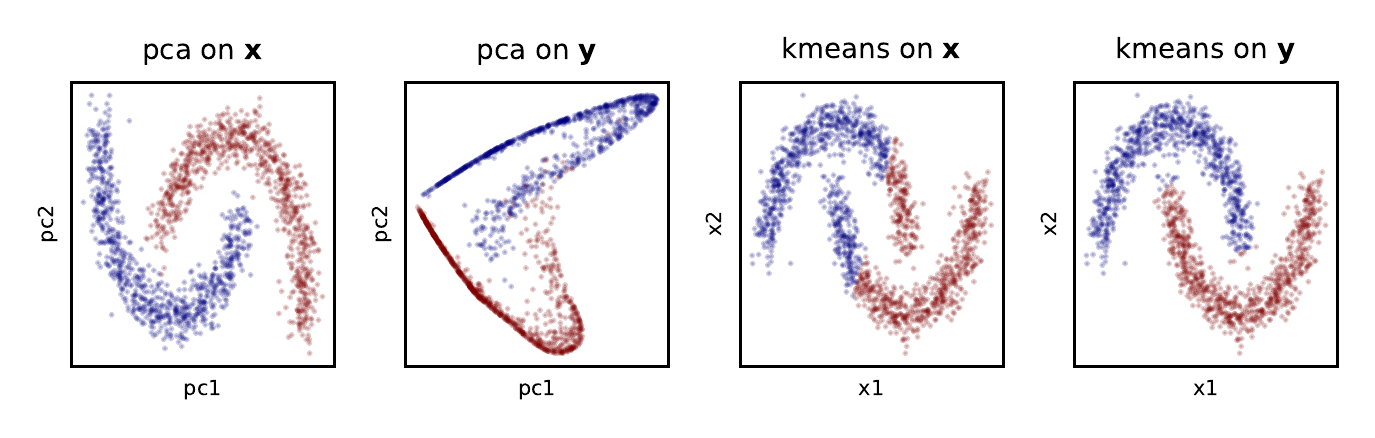}
    \caption{Utility of kernel similarity matching for downstream tasks. (a) principle components of the input vectors $\mathbf{x}$ (b) principle components of the 16 dimensional neural representations $\mathbf{y}$ (c) clustering generated by kmeans on $\mathbf{x}$ (d) clustering generated by kmeans on $\mathbf{y}$. For (a,b) the the colors are given by ground truth labels while in (c,d) the colors are given by the KMeans clustering. }
    \label{fig:half-moons-clustering}
\end{figure}

\subsection{MNIST Dataset}
We train our algorithm on the MNIST handwritten digits dataset \cite{lecun2010mnist}. The dataset consists of 70,000 images of 28x28 handwritten digits, which we cropped by 4 pixels on each side to yield 20x20 images (which become 400 dimensional inputs). We use kernels of the form $f(\mathbf{u}, \mathbf{v}) = \Vert \mathbf{u} \Vert \Vert \mathbf{v} \Vert (\hat{\mathbf{u}} \cdot \hat{\mathbf{v}})^\alpha$ and varying number of neurons. Training details are provided in the appendix.

The linear kernel is recovered by setting $\alpha=1$. We are not aware of other works using this exact ``power-cosine'' kernel before, however it is motivated by the \textit{arccosine kernel} studied in the context of wide random ReLU networks \citep{cho2009kernel}. An important property of our kernel network is the linear input-output scaling, meaning that rescaling an input $\mathbf{x}' \leftarrow a \mathbf{x}$ will cause the corresponding representation to also be rescaled by the same factor  $\mathbf{y}' \leftarrow a \mathbf{y}$. This will allow our nonlinear networks to have the same ``contrast-invariant-tuning'' properties that are thought to be displayed by simple cells in cat visual cortex \citep{skottun1987contrast}. 

\begin{figure}
    \centering
    \includegraphics[width=0.95\linewidth]{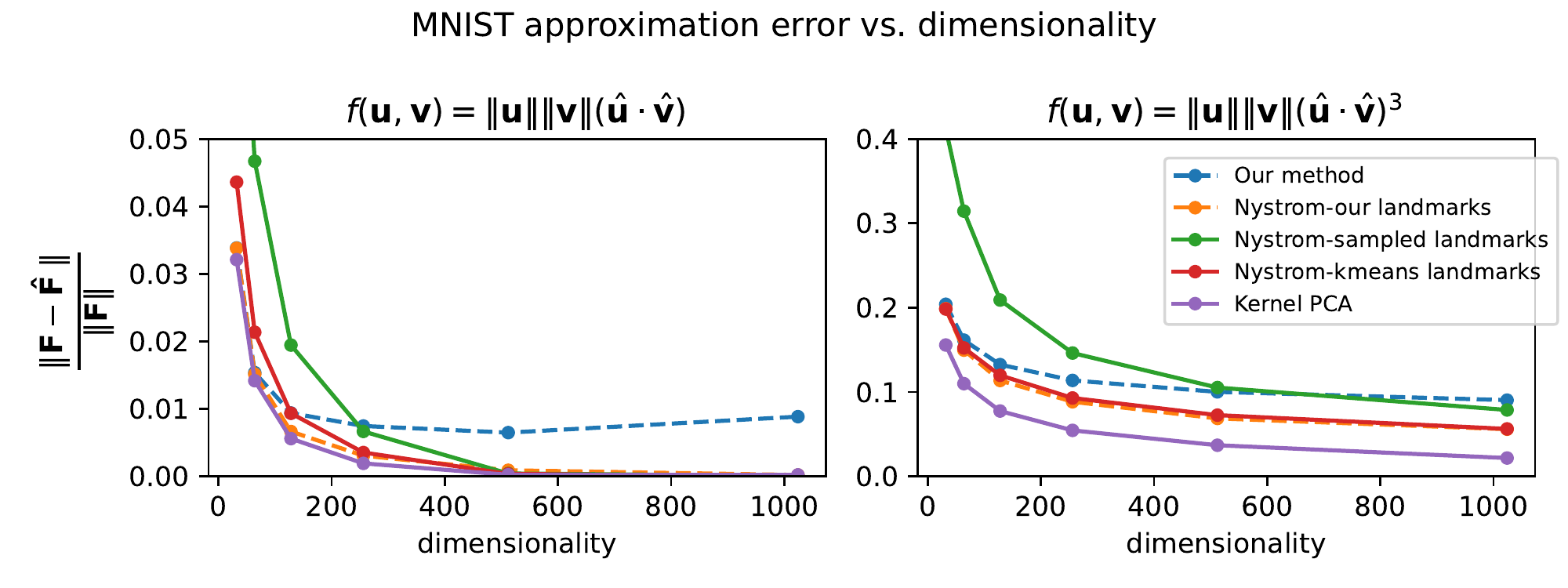}
    \caption{Approximation error vs. dimensionality for the MNIST dataset. (a) $f(\mathbf{u}, \mathbf{v}) = \mathbf{u} \cdot \mathbf{v}$  (linear kernel) (b) $f(\mathbf{u}, \mathbf{v}) = \Vert \mathbf{u} \Vert \Vert \mathbf{v} \Vert (\hat{\mathbf{u}} \cdot \hat{\mathbf{v}})^3$ (a nonlinear kernel). For the linear kernel all methods give relatively small approximation error once $n > 100$. Although yet again we see that the neural method does not continue to decrease as the dimensionality increases beyond 200, even in the linear setting. }
    \label{fig:similarities}
\end{figure}
\textbf{Approximation error}
We display the normalized approximation errors for $\alpha=1$ and $\alpha=3$ in Fig. \ref{fig:similarities}. For the linear kernel ($\alpha=1$) all methods yield a relatively small error even for low dimensionality. An error of $0.01$ is hard to see by eye when plotting input-output similarity scatter plots as done in Fig. \ref{fig:half-moons-matching}. For both $\alpha=1,3$ we observe again a strange behavior of our method: it seems to ``bottom out'' and the error stops decreasing and even begins to increase as the dimensionality increases. This may be related to unstable convergence properties of gradient descent ascent.

For $\alpha=3$ we observe that the kernel PCA method largely outperforms all methods. We obesrve that Nystr{\"o}m with either our features or KMeans appears to outperform sampled Nystr{\"o}m methods. The sampled Nystr{\"o}m method is worse than our representations for low dimensionality but eventually catches up and surpasses ours neural representations.

\begin{figure}
    \centering
    \includegraphics[width=\linewidth]{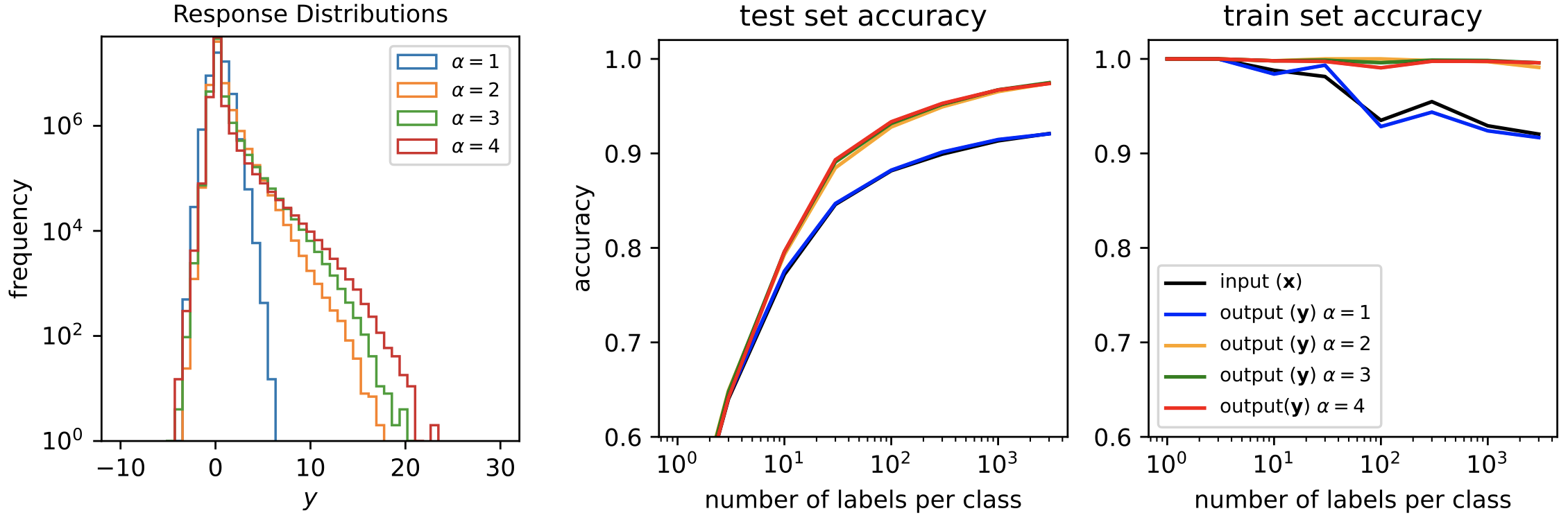}
    \caption{(a) response distribution (after the procedure we describe in the text for removing the sign degeneracy). The nonlinear kernels ($\alpha=2,3,4$) naturally give rise to sparse distributions.  (b) test set accuracy of a linear classifier classification for MNIST (c) train set accuracy of the corresponding linear classifier. Interestingly all nonlinear kernels give nearly identical train and test set results. The linear kernel gives nearly identical results to simply training the classifier directly on the pixels. }
    \label{fig:sparse_classification}
\end{figure}

\textbf{Emergence of sparse representations} We train networks with $\alpha=\{1,2,3,4\}$ and $n=800$ neurons (so the output dimensionality is exactly 2x the input dimensionality). There is a sign degeneracy when $\alpha$ is odd: we can multiply both $\mathbf{w}_i$ and $y_i$ by $-1$ and the objective is unchanged. When we look at the response histogram for single neurons, we observe that for $\alpha=3$, the response tends to be heavily skewed so that when the response has a high magnitude, it is either always positive or always negative. We remove this degeneracy by multiplying both $\mathbf{w}_i$ and $y_i$ by the sign of $\langle y_i \rangle$. After removing this degeneracy we plot the neuron responses over all patterns in Figure \ref{fig:sparse_classification}.

For $\alpha=1$ (linear neurons), neuron responses are roughly centered around zero: neuron responses are neither sparse not skewed. For $\alpha=2,3,4$, neurons appear to have a heavy tailed distribution, they frequently have small responses, but occasionally have large positive responses. Neurons become increasingly sparse and heavy tailed as we increase $\alpha$, although this effect is not that strong.

\textbf{Evaluating the representations via linear classification } We train a linear classifier on the inputs $(\mathbf{x})$ and the outputs $(\mathbf{y})$ for $\alpha=1,2,3,4$ and $n=800$. We train every configuration using $k \in \{1,3,10,30,100,300,1000,3000\}$ labels per class. We train all configuration with a weight decay parameter $\lambda \in \{1e-5,1e-4,1e-3,1e-2,1e-1,1\}$ which yields the highest test accuracy.  We average the accuracy for every configuration over 5 random seeds. The results are show in Fig. \ref{fig:sparse_classification}.

As expected, the performance of the inputs and $\alpha=1$ (linear similarity matching) is nearly identical on both test and train sets. Surprisingly, the test performance of $\alpha=2,3,4$ is nearly identical. Perhaps these curves can be partially explained by the spectra of the output similarity matrix which we show in Figure \ref{fig:eigenvalues} of the Appendix. While the shapes of the spectra are different in every case, $\alpha=1$ has roughly 200 nonzero eigenvalues while $\alpha=2,3,4$ all have nearly 800 nonzero eigenvalues. Perhaps the number of nonzero eigenvalues is more influential for the linear classification performance than the detailed shapes of these spectra.

\section{Related work}
\textbf{Kernel similarity matching with random Fourier features}
The most closely related work to ours is kernel similarity matching with random Fourier features \citep{bahroun2017neural}. The key difference between our methods is that instead of learning the features $\mathbf{w}$, they use random Fourier features to directly generate nonlinear feature vectors $\boldsymbol{\phi}^t = \sqrt{2/n} \cos(\mathbf{W} \mathbf{x}^t + \mathbf{b}) $ which they then feed into a standard linear similarity matching network. This leads them to a different architecture (one feedforward layer + one recurrent layer, instead of our single recurrent layer net) and a different set of learning rules. A benefit of the random feature approach is that it will theoretically lead to perfect matching, so long as the number of random features is sufficiently large.

However, the feature learning aspect of our algorithm naturally led to a sparse set of responses which lends our model an added degree of biological plausibility. Additionally, our method generalizes to non-shift invariant kernels and empirically it seemed that to yield better approximation error when the dimensionality of the output is not too high. Our method can be seen as a biased method for approximation, which can be useful when the dimensionality is low, but ultimately will underperform compared to non-biased methods such as random Fourier methods or Nystr{\"o}m methods.

\textbf{Nystr{\"o}m Approximation} While not obviously biological, Nystr{\"o}m methods are perhaps the most commonly used methods for approximating kernel matrices. The Nystr{\"o}m approximation uses a set of landmarks $\{\mathbf{w}_i:i=1,2,\hdots,N\}$ to construct a low rank approximation of the original kernel matrix \citep{williams2001nystrom}. To more clearly see the relationship between this method to ours, one can slightly modify the original formulation to generate ``Nystr{\"o}m features'':
\begin{equation}
    \text{``Nystr{\"o}m features'' } \;  y^t_j = \sum_{i} f(\mathbf{x}^t, \mathbf{w}_i) M_{ij}
     \text{ where } M_{ij} = 
    [(\mathbf{B}^{\dagger})^{1/2}]_{ij} \text{ and } B_{ij} = f(\mathbf{w}_i, \mathbf{w}_j)
\end{equation}
$\mathbf{B}^{\dagger}$ indicates the psuedo-inverse. Multiplying two such vectors togethers yields the Nystr{\"o}m approximation $\hat{F}_{st} = \mathbf{y}^s \cdot \mathbf{y}^t = \sum_{ij} f(\mathbf{x}^s, \mathbf{w}_i) [\mathbf{B}^{\dagger}]_{ij} f(\mathbf{x}^t, \mathbf{w}_j)$. Our method produces representations of the same functional form but our $M$ matrix is derived from parameters learned by the correlations:
\begin{equation}
     \text{Our features } \; y^t_j = \sum_{i} f(\mathbf{x}^t, \mathbf{w}_i) M_{ij}
    \text{ where } M_{ij} = 
    [\mathbf{L}+\lambda \mathbf{I}]_{ij}^{-1} q_j
\end{equation}
As measured by squared error, the Nystr{\"o}m approximation was actually a better approximation than our representations, when we used the same set of landmarks (Figs. \ref{fig:half-moons-matching}, \ref{fig:similarities}). The variation in Nystr{\"om} methods primarily come from the method used to  generate the landmarks. Two broad categories of landmark selection can be defined: template vs. pseudo-landmark.  Template based approaches choose landmarks as a subset of the inputs $\mathbf{w} \in \{\mathbf{x}_1, \mathbf{x}_w, \hdots, \mathbf{x}_T\}$ typically chosen via sampling schemes \citep{williams2001nystrom, drineas2005nystrom, musco2016recursive}. Pseudo-landmark approaches do not require the landmarks to be inputs. \citet{zhang2008improved} used the cluster centers generated by KMeans as the landmarks. \citet{fu2014landmark} formulate landmark selection as an optimization problem in the reproducing Hilbert space. Our method can be seen as a pseudo-template approach as our landmarks are directly generated via Hebbian learning rules and in general will not be exactly equal to any particular input pattern. Our method is similar in spirit to the approach of \citet{fu2014landmark}. A key difference is that we use a different objective, a correlation-based upper bound to the sum of squared errors, which gives rise to correlation-based learning rules.

\section{Discussion}
We have extended the neural random Fourier feature method of \citet{bahroun2017neural} for kernel similarity matching to instead be applicable to arbitrary differentiable kernels. Rather than using random nonlinear features, we learned the features with Hebbian learning rules. Both this work and that of \citet{bahroun2017neural} can be seen as extensions of the linear similarity matching works written in \cite{hu2014snmf, pehlevan2014nonnegative, pehlevan2015mds, pehlevan2018why}. By using a nonlinear input similarity, the representations learned by our network are capable of learning high-dimensional nonlinear functions of the input, without requiring any constraints such as non-negativity.

To our knowledge this is the first work that attempts to directly optimize the sum of squared errors (Eq. \ref{eqn:cmds}) without relying on sampling schemes or direct computation of the input similarity matrix. It would be interesting to relax the correlation-based constraint we have imposed on ourselves. This might allow for a variety of different types of bounds (Eq. \ref{eqn:inequality}) to be derived which in turn could lead to more faithful approximations than the one presented in our paper. 

A key obstacle faced by users of this algorithm is the stochastic gradient descent ascent procedure. Empirically the convergence our algorithm is quite sensitive to the learning rate choices. This method does not provide the same sorts of theoretically guarantees or empirically observed robustness of sampling based methods. Generation of more robust ascent-descent optimization methods could be useful for making this class of algorithms more accessible for the practitioner.

\section{Acknowledgements}
We would like to thank Lawrence Saul and Runzhe Yang for their helpful insights and discussions. This research was supported by the Intelligence Advanced Research Projects Activity (IARPA) via Department of Interior/ Interior Business Center (DoI/IBC) contract number D16PC0005, NIH/NIMH RF1MH117815, RF1MH123400. The U.S. Government is authorized to reproduce and distribute reprints for Governmental purposes notwithstanding any copyright annotation thereon. Disclaimer: The views and conclusions contained herein are those of the authors and should not be interpreted as necessarily representing the official policies or endorsements, either expressed or implied, of IARPA, DoI/IBC, or the U.S. Government.

\bibliographystyle{iclr2022_conference}

\appendix
\section{Appendix}

\begin{proof}[Proof of Lemma \ref{eq:KeyInequality}]
In this section we will derive a correlation-based upper bound for the nonlinear pairwise sum $\sum_{t,t'} f(\mathbf{x}_t,\mathbf{x}_{t'}) \mathbf{y}_t \cdot \mathbf{y}_{t'}$ in Equation \ref{eqn:reformulation}. The key to creating this bound will be to replace all nonlinear similarities $f(\mathbf{u}, \mathbf{v})$ with the dot product between high dimensional vectors $\boldsymbol\phi_u \cdot \boldsymbol\phi_v$. This is allowed because we have assumed that $f$ is a positive semidefinite kernel. Formally this assumption means that for any set of $M$-dimensional vectors $\{ \mathbf{w}, \mathbf{x}_1, \hdots, \mathbf{x}_T \}$, there exists a corresponding set of (at most) $T+1$-dimensional vectors $\{ \boldsymbol\phi_{\mathbf{w}}, \boldsymbol\phi_1, \hdots, \boldsymbol\phi_T \}$ whose inner products yield the similarity defined by $f$:
\begin{equation}
    \boldsymbol\phi_t \cdot \boldsymbol\phi_{t'} = f(\mathbf{x}_t, \mathbf{x}_{t'}) \text{ and } \boldsymbol\phi_t \cdot \boldsymbol\phi_{\mathbf{w}} = f(\mathbf{x}_t, \mathbf{w}) \text{ and } \boldsymbol\phi_{\mathbf{w}} \cdot \boldsymbol\phi_{\mathbf{w}} = f(\mathbf{w}, \mathbf{w}) 
    \label{eqn:hilbert-vectors}
\end{equation}
Assume we have some corresponding set of $T+1$ scalars $\{q, y_1, \hdots, y_T\}$. Now consider the vector difference $\frac{1}{T} \sum_t y_t \boldsymbol\phi_t - q \boldsymbol\phi_w$. The squared norm of this difference is of course non-negative. Additionally we can expand out this square:
\begin{equation}
    0 \leq \frac{1}{2} \left\Vert \frac{1}{T} \sum_{t} y_t \boldsymbol\phi_t - q \boldsymbol\phi_{\mathbf{w}} \right\Vert^2 = \frac{1}{2T^2} \sum_{s,t} y_s y_t \boldsymbol\phi_s \cdot \boldsymbol\phi_{t} - \frac{1}{T} \sum_t q y_t \boldsymbol\phi_t \cdot \boldsymbol\phi_{\mathbf{w}} + \frac{1}{2} q^2 \boldsymbol\phi_{\mathbf{w}} \cdot \boldsymbol\phi_{\mathbf{w}}
    \label{eqn:hilbert-norm}
\end{equation}
At this point we can simply replace all dot products with the equivalent nonlinear similarities $f(\cdot, \cdot)$ in Equation \ref{eqn:hilbert-vectors} and rearrange the terms to yield our key inequality (Eq. \ref{eqn:inequality}) which we write again:
\begin{equation}
    \frac{1}{2T^2} \sum_{s,t} y^s y^t f(\mathbf{x}^s, \mathbf{x}^t) \geq \frac{1}{T} \sum_t q y^t f(\mathbf{x}^t, \mathbf{w}) - \frac{1}{2} q^2 f(\mathbf{w}, \mathbf{w})
    \label{eqn:inequality-again}
\end{equation}
\end{proof}

\subsection{Interpretation using rank-1 Nystr{\"o}m approximation}
The bound in Equation \ref{eqn:inequality} can be interpreted using a rank-1 Nystr{\"o}m approximation for $f(\mathbf{x}_t,\mathbf{x}_{t'})$. By holding $\mathbf{w}$ fixed and maximizing for $q$ in the right hand side of Equation \ref{eqn:inequality}, we get $q^* = f(\mathbf{w}, \mathbf{w})^{\dagger} \sum_t y_t f(\mathbf{x}_t, \mathbf{w})$ where $f(\mathbf{w}, \mathbf{w})^{\dagger}$ indicates the pseudo-inverse.\footnote{The pseudo-inverse of the scalar $f(\mathbf{w}, \mathbf{w})$ acts exactly like the regular inverse except it is defined to be $0$ when $f(\mathbf{w}, \mathbf{w})$ is zero, unlike the regular inverse which would be undefined.} We can insert this optimal $q^*$ back into the right hand side to yield:

\begin{equation}
    \sum_t q^* y_t f(\mathbf{x}_t, \mathbf{w}) - \frac{1}{2} (q^*)^2 f(\mathbf{w}, \mathbf{w}) = \frac{1}{2} \sum_{t,t'} y_t f^{\text{Nystr{\"o}m}}_{t,t'} y_{t'}
\end{equation}
where we have defined they Nystr{\"o}m matrix:
\begin{equation}
    f^{\text{Nystr{\"o}m}}_{t,t'} := f(\mathbf{x}_t, \mathbf{w}) f(\mathbf{w}, \mathbf{w})^{\dagger} f(\mathbf{w}, \mathbf{x}_{t'})
\end{equation}
The matrix $f^{\text{Nystr{\"o}m}}_{t,t'}$ is a rank-1 Nystr{\"o}m approximation for the full similarity matrix $f(\mathbf{x}_t, \mathbf{x}_{t'})$ \cite{williams2001nystrom}. Note that for every dimension $i$ of the representation vector $\mathbf{y}$, we have a different landmark vector $\mathbf{w}^i$ so we are using a different rank-1 approximation of the matrix $f(\mathbf{x}_t, \mathbf{x}_{t'})$ for every to the pairwise sum $\frac{1}{2} \sum_{t,t'} y^i_t y^i_{t'} f(\mathbf{x}_t, \mathbf{x}_{t'})$.

Typically the weight vector $\mathbf{w}$ , often called a ``landmark'', used in the Nystr{\"o}m approximation is set either by setting it to a random input or by more sophistcated schemes like setting it with KMeans. In our case, we are directly optimizing the landmarks via Equation \ref{eqn:upper_bound}. To our knowledge the only other work to do this was performed in \cite{fu2014landmark}.

\subsection{PyTorch code for training}
\begin{figure}
    \centering
    \includegraphics[width=0.75\linewidth]{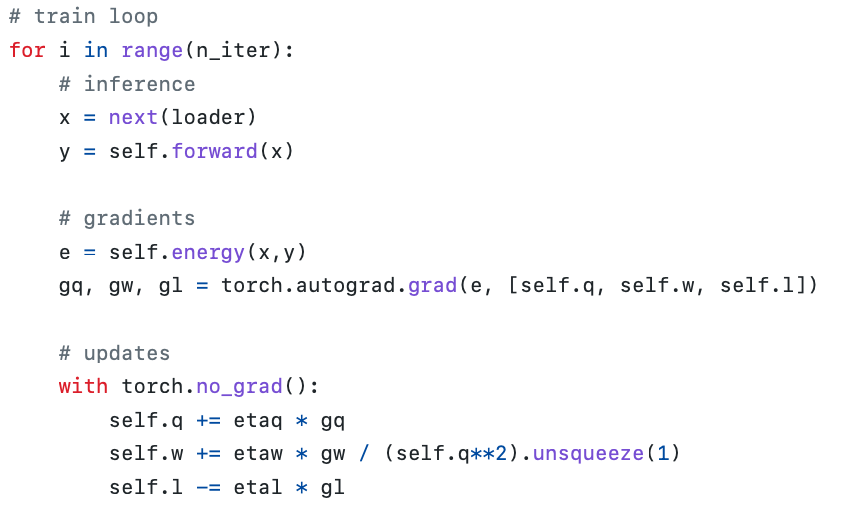}
    \caption{Training loop to perform the GDA-based optimazation of Eq. \ref{eqn:final_objective} written using PyTorch}
    \label{fig:train-loop}
\end{figure}
The code used in the main training loop of our algorithm is shown in Fig. \ref{fig:train-loop}.

\subsection{Homogeneous kernels}
\label{sec:homogeneous-kernel}
Before moving on we note that a simplification can be made  if we have a homogenous (scale free) kernel, i.e. if $f(a \mathbf{u}, b \mathbf{v}) = (ab)^\alpha f(\mathbf{u}, \mathbf{v})$. Examples of such kernels are the linear kernel $f(\mathbf{u}, \mathbf{v}) = \mathbf{u} \cdot \mathbf{v}$, homogeneous polynomial kernels $f(\mathbf{u}, \mathbf{v}) = (\mathbf{u} \cdot \mathbf{v})^\alpha$, and the cosine-based kernel we will use in one of our experiments $f(\mathbf{u}, \mathbf{v}) = \Vert \mathbf{y} \Vert \Vert \mathbf{v} \Vert (\hat{\mathbf{u}} \cdot \hat{\mathbf{v}})^\alpha$. In this case, there is a degenecary between $q_i$ and the norm of $\mathbf{w}_i$. This means we can actually eliminate the minimization over $\mathbf{q}$ by setting $q_i=1$. We prove this fact in the appendix. In the case of a homogeneous kernel, we are left with the simpler equivalent optimization:
\begin{equation}
    \min_{\mathbf{W}}  \max_{\mathbf{L}} \min_{\mathbf{Y}} -\sum_{i=1}^N \left[ \langle y_i f(\mathbf{w}_i, \mathbf{x}) \rangle - \frac{1}{2} f(\mathbf{w}_i,\mathbf{w}_i) \right] + \frac{1}{2} \sum_{i,j=1}^N \left[ L_{ij} \langle y_i y_j \rangle - \frac{1}{2} L_{ij}^2 \right] + \lambda \langle y_i^2 \rangle
    \label{eqn:scalefree-objective}
\end{equation}

\begin{equation}
    \min_{\mathbf{W}}  \max_{\mathbf{L}} \min_{\mathbf{Y}} \frac{1}{T} \sum_{t=1}^T\left[-\sum_{i=1}^N \left[ y^t_i f(\mathbf{w}_i, \mathbf{x}^t) - \frac{1}{2} f(\mathbf{w}_i,\mathbf{w}_i) \right] + \frac{1}{2} \sum_{i,j=1}^N \left[ L_{ij} \langle y_i y_j \rangle - \frac{1}{2} L_{ij}^2 \right] + \frac{\lambda}{2} \sum_{i=1}^N (y^t_i)^2 \right]
    \label{eqn:scalefree-objective}
\end{equation}
This more clearly shows the relationship between the linear similarity matching objectives and the more general kernel similarity matching objective. When $f(\mathbf{w}_i,\mathbf{x}) = \mathbf{w}_i \cdot \mathbf{x}$, we are in fact left with the exact objective studied in previous works on linear similarity matching \cite{pehlevan2018why}. This simplification can be easily implemented in code by initializing $q_i=1$ and setting the learning rate for $\eta_q$ to be 0 for all iterations.

\subsection{Proof that the bound in Equation \ref{eqn:upper_bound} is maximized when $q=1$ and $f$ is a homogeneous kernel}
Assume that $f$ is a homogenous kernel, so that $f(\lambda_1 \mathbf{u}, \lambda_2 \mathbf{v}) = (\lambda_1 \lambda_2)^{\alpha} f(\mathbf{u}, \mathbf{v})$ for any $\lambda_1, \lambda_2 > 0$. We will show that in this case we can simply set $q=1$. Assume we have some pair $q,\mathbf{w}$. Then define $q'=1$ and $\mathbf{w}' := q^{1/\alpha} \mathbf{w}$. Because our kernel is homogeneous, we have $f(\mathbf{w}', \mathbf{x}_t) = f(q^{1/\alpha} \mathbf{w},\mathbf{x}_t) = q f(\mathbf{w},\mathbf{x}_t)$ and similarly $f(\mathbf{w}',\mathbf{w}') = q^2 f(\mathbf{w},\mathbf{w})$. In other words when we have a homogenous kernel, we can always just rescale the features $\mathbf{w}' \leftarrow q^{1/\alpha} \mathbf{w}$ so the following holds for any $q$:
\begin{equation}
    \sum_t q y_t f(\mathbf{x}_t, \mathbf{w}) - \frac{1}{2} q^2 f(\mathbf{w}, \mathbf{w}) = \sum_t y_t f(\mathbf{x}_t, \mathbf{w}') - \frac{1}{2} f(\mathbf{w}', \mathbf{w}')
    \label{eqn:eliminate_q}
\end{equation}

\subsection{Methods we compare to in our experiments}
\textbf{Kernel PCA} The optimal (in terms of mean squared error) rank $N$ approximation $\hat{f}$ to the kernel matrix $f$ is given by the top $n$-dimensional subspace of the kernel matrix \citep{borg2005cmds}. Specifically, we perform an eigenvector decomposition on $f$ then set $\hat{f}$ via:
\begin{equation}
    f(\mathbf{x}_s, \mathbf{x}_t) = \sum_{i=1}^T \lambda_i \mathbf{v}_i \mathbf{v}_i^{\top} \rightarrow  \hat{f}_{st} = \sum_{i=1}^N \lambda_i \mathbf{v}_i \mathbf{v}_i^{\top}
\end{equation}
For the mnist dataset, the kernel matrix is 70k x 70k entries so we use a randomized svd algorithm to compute the top components, rather than a full SVD. We use the PyTorch implementation ``torch.svd\_lowrank'' with $q=1024+256$ (so we estimate the top 1024+256 singular values and vectors) and we set niter=4 meaning we do 4 power iterations. 

\textbf{Nystrom methods} Given a set of landmarks $\{\mathbf{w}_i:i=1,2,\hdots,N\}$, the nystrom method defines two matrices:
\begin{equation}
    A_{ti} = f(\mathbf{x}^t,\mathbf{w}_i) \;\;\;\;\;\; B_{ij} = f(\mathbf{x}_i,\mathbf{w}_j)
\end{equation}
These are used to approximate the kernel matrix via:
\begin{equation}
    \hat{f}_{st} = \sum_{ij} A_{si} [\mathbf{B}^{\dagger}]_{ij} A_{tj}^{\top}
\end{equation}
To calculate the pseudo-inverse of $\mathbf{B}$ we use double precision arithmetic and set first set all singular value of $B$ smaller than $1e-10$ to zero. We compare 3 different methods of landmark generation in our paper.

\textbf{Nystrom with uniformly sampled landmarks} This is the simplest method, and was proposed in he original paper using the Nystrom method to approximate kernel matrices \citep{williams2001nystrom}. We simply uniformly sample $N$ landmarks without replacement from the dataset.

\textbf{Nystrom with landmarks generated via KMeans} This method was used by \citet{zhang2008improved} and instead uses the cluster centers given by KMeans as the landmarks. We initialize our means with templates from the dataset and the use Lloyd's method to update our cluster centers \citep{lloyd1982least}. This is run either until convergence, or 100 iterations of the algorithm occurs, whichever happens first.

\textbf{Nystrom with landmarks generated via our method} We use the $N$ features learned with Hebbian update rules as the landmarks in thye Nystrom approximation. 

\textbf{Random Fourier features} For the half moons dataset using the Gaussian kernel, we also compare our method to the random Fourier feature method \citep{rahimi2007random}. The authors in \citet{bahroun2017neural} train a linear similarity matching on top of these features. But for simplicity, we just use the random features themselves, rather than the subsequent neurally generated features. This provides a best-case scenario for the neural random Fourier method. The neural algorithm is simply trying to matching similarities $\min{y} \Vert \mathbf{y}^s \cdot \mathbf{y}^t - \boldsymbol{\phi^s} \cdot \boldsymbol{\phi}^t \Vert$, and it should be able to provide zero error, given the same output dimension as input dimension. Although it practice it can be challenging to set the learning rates appropriately, so we evaluate $\boldsymbol{\phi}$ instead of $\mathbf{y}$ to avoid any possible issues with improper training.

To generate these features, we randomly sample $\mathbf{w}_i \sim \mathcal{N}(\text{mean} =0,\text{variance}=\frac{1}{\sigma^2}\mathbf{I})$ where and $b_i \in \text{Uniform}[0,2 \pi]$ as set the features as 
\begin{equation}
    \phi^t_i = \sqrt{\frac{2}{n}} \cos (\mathbf{w}_i \cdot \mathbf{x}^t + b_i)
\end{equation}

\subsection{Half Moons Experiment}
\subsubsection{Training details}
We train with minibatch sizes of 64 input. We train for 10000 iterations with $\eta_w=\eta_q=0.01$ and $\eta_l=0.1$. Then we anneal the learning rates by a factor of 10x and train for 10000 more iterations $\eta_w=\eta_q=0.001$ and $\eta_l=0.01$. 

\subsection{MNIST Experiment}
\subsubsection{Training details}
We train with minibatch sizes of 64.  After initialization and warmup, we set $\alpha$ and train for 10,000 iterations with $\eta_w=0.001, \eta_l=0.01$. We then decay the learning rates for $W,L$ by 10x and train for with 5k more iterations with $\eta_w=0.0001, \eta_l=0.001$. This whole procedure takes approximately 4 minutes on an NVIDIA GTX 1080 GPU.

\subsubsection{Learned features for MNIST dataset}
\begin{figure}
    \centering
    \includegraphics[width=\linewidth]{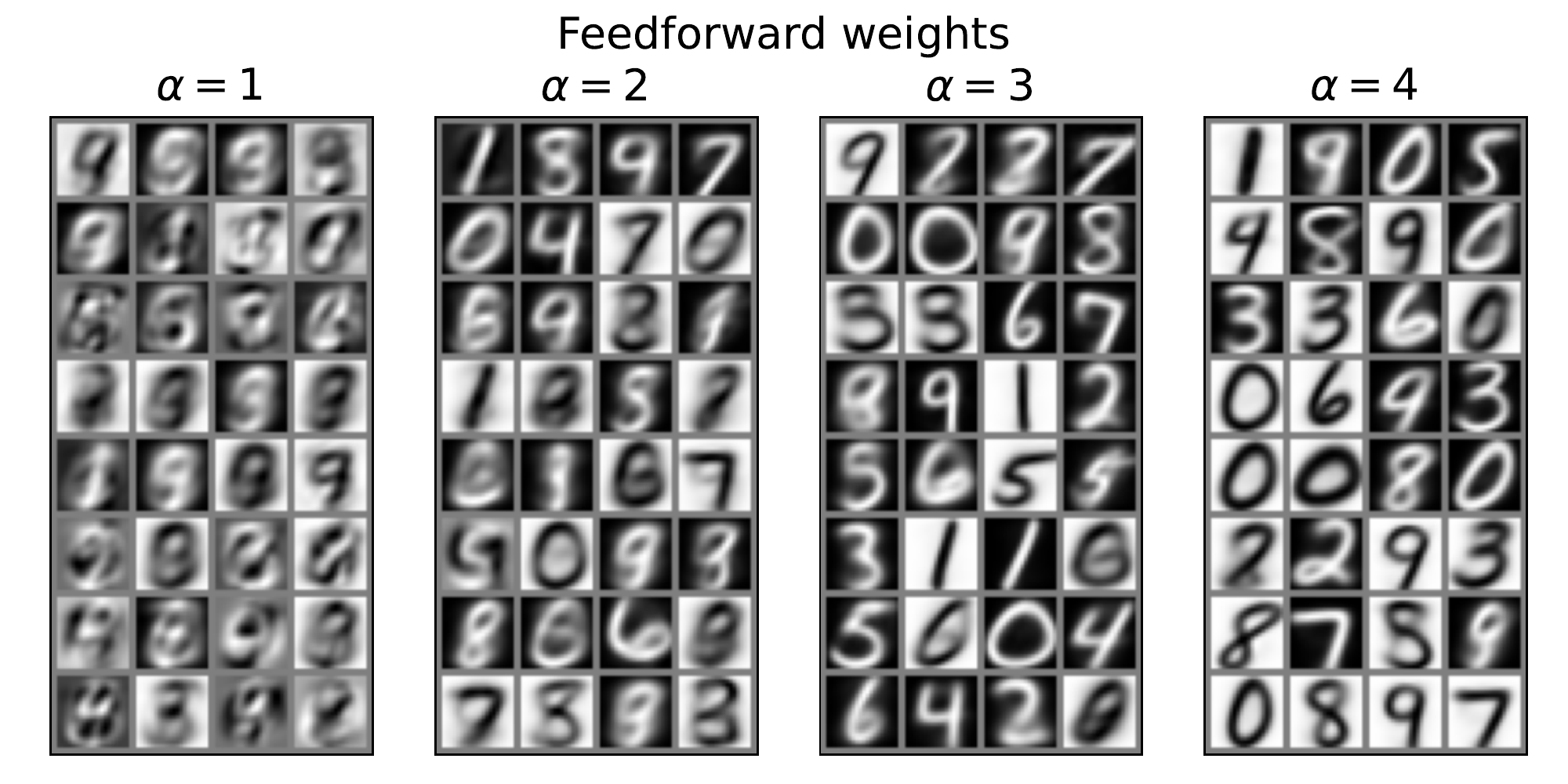}
    \caption{Feedforward weights $(\mathbf{W})$ learned by the network for $\alpha=1,2,3,4$. When $\alpha=1$, the weights appear to be complicated linear combinations of input vectors. As $\alpha$ increases, the weights begin to resemble ``templates'', i.e. whole digits. In the main text, we argue this behavior results from the increasing sharpness of neural tuning as $\alpha$ increases.}
    \label{fig:feedforward-weights}
\end{figure}
In Figure \ref{fig:feedforward-weights} we show the weights $\mathbf{w}_i$, visualized as 20x20 images, that are learned. When $\alpha=1$ (linear similarity matching), the features appear as complicated linear combinations of input digits. However, with $\alpha=2$ we see clear digits beginning to emerge. And with $\alpha=4$ nearly all the features look like whole digits. 

\subsection{Receptive field analysis (aka "linearized neuron responses")}
\begin{figure}
    \centering
    \includegraphics[width=\linewidth]{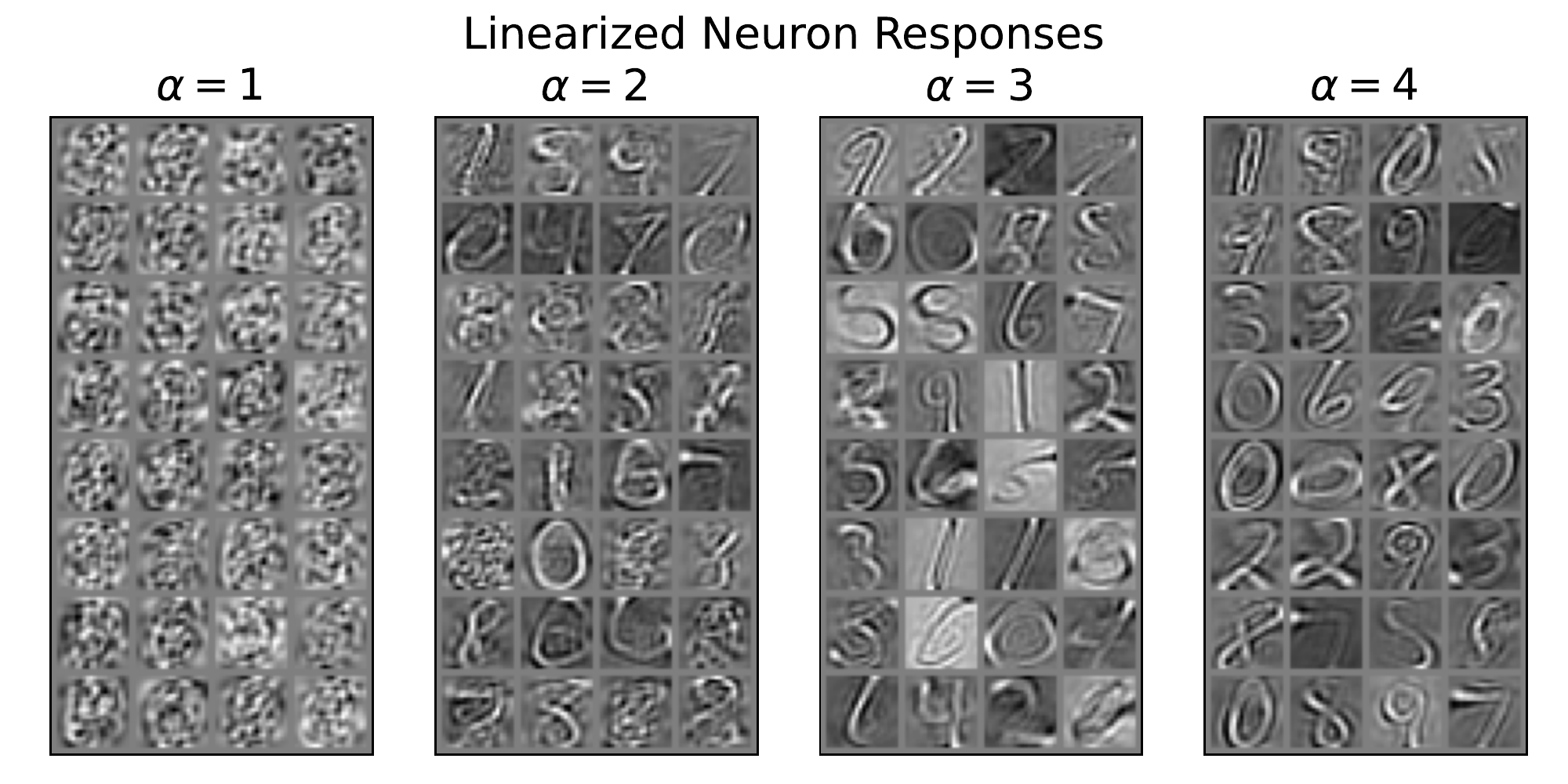}
    \caption{``Linearized responses'' for a subset of neurons from networks with $\alpha=\{1,2,3,4\}$. Specifically, for each neuron $y_i$ we compute the vector $\mathbf{s}_i = \left[ 0.1\; \mathbf{I} + \langle \mathbf{x} \mathbf{x}^{\top} \rangle\right]^{-1} \langle y_i \mathbf{x} \rangle$ and visualize $\mathbf{s}_i$ as a $20\times20$ image. As $\alpha$ increases, it appears that neurons become increasingly selective to whole input digits.}
    \label{fig:linearized-responses}
\end{figure}
A natural way to visualize what the networks learn is to examine the feedforward weights. However these visualizations are not as interpretable in this experiment as they were for the simple halfmoons dataset. In particular for $\alpha=1,2,3$ the weights appear to be a blend of templates (whole digits) and more complicated linear combinations of digits. We show some examples from each network configuration in the appendix.

We can better understand and visualize the network responses by instead examining the linearized neuron responses. Specifically, for each neuron $y_i$ we compute the vector $\mathbf{s}_i = \left[ 0.1\; \mathbf{I} + \langle \mathbf{x} \mathbf{x}^{\top} \rangle\right]^{-1} \langle y_i \mathbf{x} \rangle$. This vector can be thought of as a linear approximation to each neuron $y_i \approx \mathbf{s}_i \cdot \mathbf{x}$. We show these vectors, again visualized as images, in Figure \ref{fig:linearized-responses}. 

These linearized responses actually highlight a behavior not seen by only considering feedforward weights. We see for $\alpha=2$, it appears that many of the neurons appear to be selective for smaller regions of the input, sometimes interpretable as strokes and edges. This behavior is likely coming from some sort of cancellation between the feedforward input and lateral interactions. As $\alpha$ increases, the linear filters appear to grow in size to resemble whole digits.

For $\alpha=1$ (aka linear similarity matching) the linearized responses do not in any way appear as whole digits, rather they appear to be high spatial frequency images. This is not a failure of the networks, as the input-output similarities are nearly perfectly matched. This behavior results from the fact that linearity is not enough to encourage parts or whole templates to be learned.

\subsubsection{Spectral analysis of the representations}
\begin{figure}
    \centering
    \includegraphics[width=0.35\linewidth]{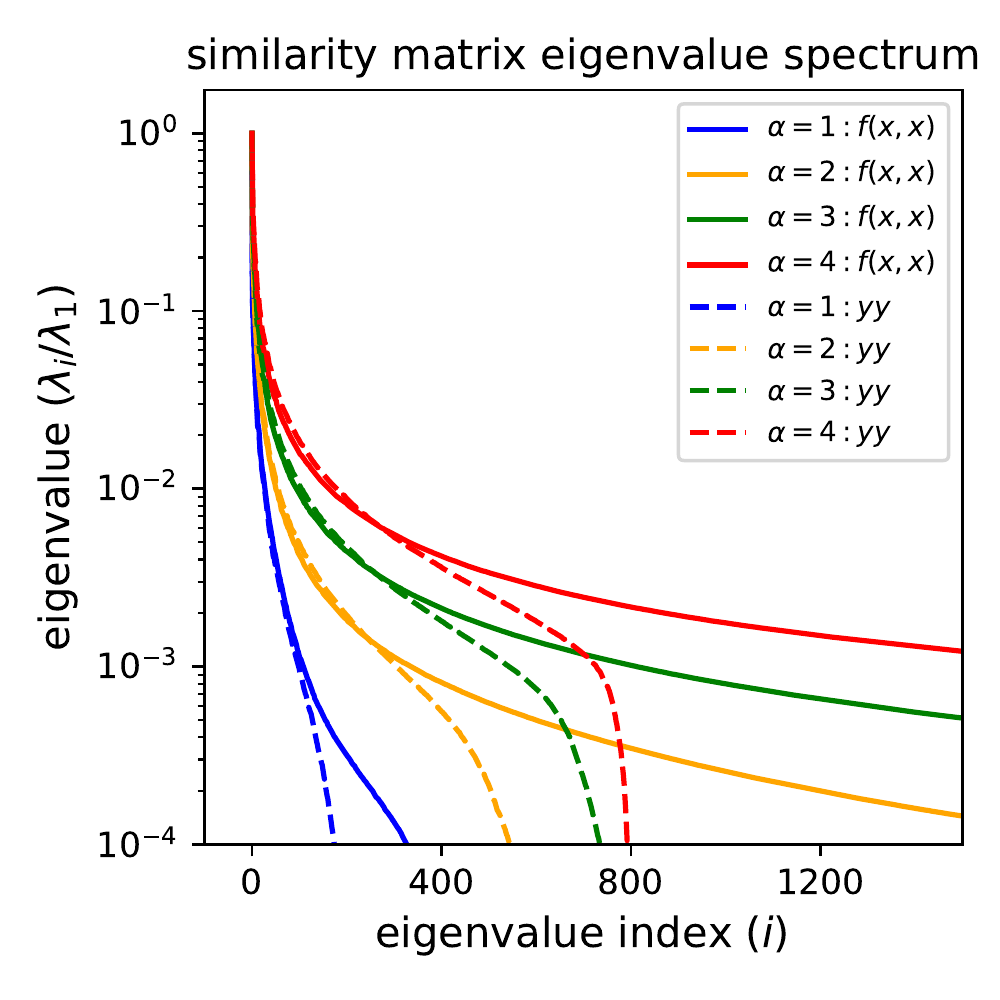}
    \caption{Eigenvalue spectrum of the input similarity matrix $f(\mathbf{x}_t, \mathbf{x}_{t'})$ and learned output similarity matrix $\mathbf{y}_t \cdot \mathbf{y}_{t'}$. If similarity matching were optimal, (i.e. we just performed uncentered kernel pca on the input similarity matrix) the largest 800 eigenvalues would be exactly matched and subsequent eigenvalues would be zero. We see that increasing $\alpha$ brings up the tails of the spectrum, approximately "whitening" the responses. For $\alpha=1$, because the inputs are 400 dimensional, the spectrum only has at most 400 nonzero eigenvalues.}
    \label{fig:eigenvalues}
\end{figure}
We examine the eigenvalue spectrum of the input similarity matrix $f(\mathbf{x}_t, \mathbf{x}_{t'})$ and the output similarity matrix $\mathbf{y}_t \cdot \mathbf{y}_{t'}$. We plot these spectra in Figure \ref{fig:eigenvalues}. Note that we normalize the spectra by dividing by the largest eigenvalue.

Without even considering the output representations, we can already observe interesting behavior just by considering the spectrum of the input similarity matrix. As we increase $\alpha$, the ``sharpness'' of the kernel, the spectrum of $f$ tends to flatten out. The effective rank of this matrix increases with increasing kernel sharpness. This observation is is in part a motivation for kernel similarity matching. Matching a high rank matrix naturally requires high dimensional vectors. This increase in dimensionality may be useful for downstream tasks such as linear classification. It is also an important part of brain inspired modeling to use overcomplete representations of the input \cite{olshausen1997overcomplete}.

For $\alpha=1$, the spectrum of $\mathbf{y}_t \cdot \mathbf{y}_{t'}$ closely matches the spectrum $f(\mathbf{x}_t, \mathbf{x}_{t'})$ for the larger eigenvalues. However, it appears to fall off for smaller eigenvalues. This may be due in part to the training not being fully converged. For $\alpha>1$, the spectrum of $\mathbf{y}_t \cdot \mathbf{y}_{t'}$ approximately matches for the larger eigenvalues (although not perfectly). However again the spectrum tends to fall off more rapidly for the learned representations than for the input similarity matrix. Note that because the dimensionality of $\mathbf{y}$ is 800 for all experiments, the spectrum necessarily must be zero for all eigenvalues smaller than the 800th largest eigenvalue.

\end{document}